\def\year{2021}\relax
\documentclass[letterpaper]{article} 
\usepackage{aaai21}  
\usepackage{times}  
\usepackage{helvet} 
\usepackage{courier}  
\usepackage[hyphens]{url}  
\usepackage{graphicx} 
\usepackage{natbib}  
\usepackage{caption} 
\usepackage{caption}
\usepackage{environ}
\usepackage[colorinlistoftodos]{todonotes}
\usepackage{algorithm}
\usepackage[switch]{lineno}
\usepackage{tikz-cd}

 \usepackage{amsthm}
 \usepackage{amsmath,amsfonts,amssymb}  

\usepackage{xcolor}
\usepackage{graphicx}
\usepackage{subfigure}
\usepackage{algpseudocode}

\newtheorem*{lem*}{Lemma}
\newtheorem*{def*}{Definition}

\newtheorem{rem}{Remark}

\newtheorem*{theo*}{Theorem}
\newtheorem{theo}{Theorem}
\newtheorem{lem}{Lemma}
\newtheorem{ass}{Assumption}

\title{Mind2Mind : Transfer Learning for GANs}

\author{
    Yael Fregier, 
    Jean-Baptiste Gouray \thanks{both authors contributed equally}
    \\
}
\affiliations{

     Univ. Artois, UR2462\\ Laboratoire de Math\'ematiques de Lens (LML)\\ F-62300 Lens, France\\
yael.fregier@univ-artois.fr, jeanbaptiste.gouray@gmail.com 
}

\begin{document}

\maketitle

\begin{abstract}
Training generative adversarial networks (GANs) on high quality (HQ) images involves important computing resources. This requirement represents a bottleneck for the development of applications of GANs. We propose a transfer learning technique for GANs that significantly reduces training time. Our approach consists of freezing the low-level layers of both the critic and generator of the original GAN. We assume an auto-encoder constraint in order to ensure the compatibility of the internal representations of the critic and the generator. This assumption explains the gain in training time as it enables us to bypass the low-level layers during the forward and backward passes. We compare our method to baselines and observe a significant acceleration of the training. It can reach two orders of magnitude on HQ datasets when compared with StyleGAN. We prove rigorously, within the framework of optimal transport, a theorem ensuring the convergence of the learning of the transferred GAN. We moreover provide a precise bound for the convergence of the training in terms of the distance between the source and target dataset.  \end{abstract}

\section{Introduction}

The recent rise of deep learning as a leading paradigm in AI mostly relies on computing power (with generalized use of GPUs) and massive datasets. These requirements represent bottlenecks for most practitioners outside of big labs in industry or academia and are the main obstacles to the generalization of the use of deep learning.  Therefore, methods that can bypass such bottlenecks are in strong demand.  Transfer learning is one of them and various methods of transfer learning (for classification tasks) specific to deep neural networks have been developed \cite{Tan2018ASO}.

 A {\it generative problem} is a situation in which one wants to be able to produce elements that could belong to a given data set $\mathcal{D}$.
Generative Adversarial Networks (GANs) were introduced in 2014 \citep{Goodfellow:2014aa} \cite{Salimans:2016aa}  to tackle generative tasks with deep learning architectures and improved in \cite{Arjovsky:2017aa, Gulrajani:aa} (Wasserstein GANs ). They immediately took a leading position in the world of generative models, comforted by progress with HQ images \cite{ProGAN}, \cite{StyleGAN}.
The goal of our work (see section \ref{section::algo}) is to develop in a generative setting, i.e., for GAN architectures, the analog of the {\it cut-and-paste} approach. 

The main idea of our method is to reuse, for the training of a GAN, the weights of an autoencoder already trained on a source dataset $\mathcal{D}$.  The weights of the low level layers of the generator (resp. critic) will be given by those of the decoder (resp. encoder). We call {\it MindGAN} the high level layers of the generator. It is a subnetwork that is trained as a GAN on the encoded features of the target dataset $\mathcal{D'}$.

We prove in section \ref{j'te l'garantis} a theorem that controls the convergence of the transferred GAN in terms of the quality of the autoencoder, the domain shift between $\mathcal{D}$ and $\mathcal{D'}$ and the quality of the MindGAN. As a consequence, our experimental results in section \ref{results} rely heavily on the choice of the autoencoder. ALAE autoencoders \cite{ALAE} are extremely good autoencoders that turned out to be crucial in our experiments with HQ images. Their use, in conjunction with our transfer technique, enables an {\bf acceleration} of the training {\bf by a factor of 656}, while keeping a good quality. 




\section{Preliminaries}\label{WGANS}

A GAN consists of two networks trained adversarially. The {\it generator} $g:Z\rightarrow \chi$ associates to a vector $z$ sampled from a latent vector space $Z$ a vector $g(z)$ in another vector space $\chi$ while the {\it discriminator} $c:\chi \rightarrow \mathbb{R}$ learns to associate a value close to 1 if the vector $g(z)$ belongs to $\mathcal{D}$ and zero otherwise. Their respective loss functions, $L_g$ and $L_c$ are recalled in section \ref{section::algo}.

One can assume that elements of $\mathcal{D}$ can be sampled from an underlying probability distribution $\mathbb{P}_\mathcal{D}$ on a space $\chi$ and try to approximate it by $\mathbb{P}_\theta$, another distribution on $\chi$ that depends on some learnable parameters $\theta$.  {\it Generating} then means sampling from the distribution $\mathbb{P}_\theta$. The main idea behind a Wasserstein GAN is to use the {\it Wasserstein distance} (see appendix \ref{proof!}, and \cite{Villani} definition 6.1) to define by $W(\mathbb{P}_\mathcal{D},\mathbb{P}_\theta)$ the loss function for this optimisation problem.  More precisely, the Wasserstein distance is a distance on Borel probability  measures on $\chi$ (when compact metric space). In particular, the quantity $W(\mathbb{P}_\mathcal{D},\mathbb{P}_\theta)$ gives a number which depends on the parameter $\theta$ since $\mathbb{P}_\theta$ depends itself on  $\theta$. The main result of \cite{Arjovsky:2017aa} asserts that
if $\mathbb{P}_\theta$ is of class $\mathcal{C}^k$ as a function of $\theta$ almost everywhere, $W(\mathbb{P}_D, \mathbb{P}_\theta)$ is also of class $\mathcal{C}^k$ with respect to $\theta$. As a consequence, one can solve this optimization problem by doing gradient descent for the parameters $\theta$ (using a concrete gradient formula provided by the same theorem) until the two probability distributions coincide.

Among the distributions on $\chi$, some can be obtained from a prior distribution $\mathbb{P}_{Z}$ on an auxiliary latent space $Z$ and a map $g : Z \rightarrow \chi$ as follows. The {\it push-forward} of $\mathbb{P}_{Z}$ under $g$ \cite{Bo} is defined so that a sample is given by $g(z)$ the image through $g$ of a sample $z$ from the distribution $\mathbb{P}_{Z}$. We will denote this pushforward by ${g}_\sharp \mathbb{P}_{Z}$ and when $g$ depends on parameters $\theta$ use instead the notation $\mathbb{P}_\theta:={g}_\sharp \mathbb{P}_{Z}$. In practice, one can choose for $\mathbb{P}_{Z}$ a uniform or Gaussian distribution, and for $g$ a (de)convolution deep neural network. In our applications, we will consider for instance $Z:=\mathbb{R}^{128}$ equipped with a multivariate gaussian distribution and $\chi=[-1,1]^{28\times 28}$, the space of gray level images of resolution $28\times 28$. Hence, sampling from $\mathbb{P}_\theta$ will produce images.

In order to minimise the function $W(\mathbb{P}_\mathcal{D},\mathbb{P}_\theta)$, one needs a good estimate of the Wasserstein distance. The {\it Rubinstein-Kantorovich duality} (See \cite{Villani} theorem 5.9) states that $W(\mathbb{P}, \mathbb{P}')=\max_{\vert c\vert_L\leq 1} \mathbb{E}_{x \sim \mathbb{P}} \ c(x)-\mathbb{E}_{x\sim \mathbb{P}'} \ c(x)$, where $\mathbb{E}_{x \sim \mathbb{P}}f(x)$ denotes the expected value of the function $f$ for the probability measure $\mathbb{P}$, while the max is taken on the unit ball for the Lipschitz semi-norm. Concretely, this max is obtained by gradient ascent on a function $c_\theta$ encoded by a deep convolution neural network.



In our case, when $\mathbb{P}':=\mathbb{P}_\theta$, the term $\mathbb{E}_{x\sim \mathbb{P}'} \ c(x)$ takes the form $\mathbb{E}_{z\sim \mathbb{P}_Z} \ c_\theta(g_\theta(z))$. One recovers the diagram \begin{equation}\label{pair}Z\overset{g_\theta}{\longrightarrow} \chi \overset{c_\theta}\longrightarrow \mathbb{R}\end{equation} familiar in the adversarial interpretation of GANs. With this observation, one understands that one of the drawbacks of GANs is that there are two networks to train. They involve many parameters during the training, and the error needs to backpropagate through all the layers of the two networks combined, i.e., through $c_\theta\circ g_\theta$. This process is computationally expensive and can trigger the vanishing of the gradient. Therefore, specific techniques need to be introduced to deal with very deep GANs, such as in \cite{growing}. The approach we present in section \ref{mind2mind} can circumvent these two problems.


\label{Jaimelatrans}

 {\it Transfer learning} is a general approach in machine learning to overcome the constraints of the volume of data and computing power needed for training models. It leverages a priori knowledge from a learned task $\mathcal{T}$ on a source data set $\mathcal{D}$ in order to learn more efficiently a task $\mathcal{T}'$ on a target data set $\mathcal{D}'$. It applies in deep learning in at least two ways: {\it Cut and Paste} and {\it Fine tuning}.

{\bf Cut and Paste} takes advantage of the difference between high and low-level layers of the network $c$. It assumes that the network $c$ is composed of two networks $c_0$ and $c_1$ stacked one on each other, i.e., mathematically that $c=c_0\circ c_1$ (one understands the networks as maps). While the low-level layers $c_1$ process low-level features of the data, usually common for similar datasets, the high-level layers $c_0$ are in charge of the high-level features which are very specific to each dataset. Hence, instead of retraining all the weights of an auxiliary network pre-trained on $\mathcal{D}$, one can retrain only the parameters of the last layers of the network while keeping the other parameters untouched. This approach boils down to the following steps \label{steps} :

\begin{enumerate}
\item  identify a dataset $\mathcal{D}$ similar to the data set $\mathcal{D}'$ we are interested in, both in the same space $\chi$,
\item  import a network $c=c_0\circ c_1$ , already trained on the dataset $\mathcal{D}$, where $c$ factors as 
\begin{tikzcd}[column sep=1.5em]
 &M \arrow{dr}{c_0} &\\
\chi \arrow{ur}{c_1} \arrow{rr}{c} && C \end{tikzcd} with $C$ the space of classes.

\item pass the new dataset $\mathcal{D}'$ through $c_1$ and train $c_0'$ on $c_1(\mathcal{D}')$, and
\item use $c':=c_0'\circ c_1$ as the new classifier on $\mathcal{D}'$.
\end{enumerate}

The main advantages of this approach are the following :
\begin{itemize}
\item[a.] much fewer parameters to train (only the parameters of $c_0'$, which in practice correspond to a few dense layers),
\item[b.] need to pass the data $\mathcal{D}'$ only once through $c_1$, and
\item[c.] no need to backpropagate the error through $c_1$.

\end{itemize}

 Since the low-level features often represent the most time-consuming part of the training, eliminating the need to train their weights will accelerate the process. If the datasets are similar, only training the last layers generally leads to good results in a much shorter time and with fewer data.
 
  {\bf Fine tuning} is based on the same first two steps than {\it Cut and Paste}, but instead of steps 3 and 4, uses the weights of $c$ to initialise the training of the new network $c'$ on $\mathcal{D}'$. It is assumed that $c$ and $c'$ share the same architecture.


In both approaches, the network $c$ must have been previously trained by a third party during a very long time on the source dataset $\mathcal{D}$, which is potentially much more massive than $\mathcal{D}'$. 
It turns out in practice that these approaches enable to train a network on a new task with much less computing power and data (see \cite{donahue:a}).

\section{Related works}
In the following section, we survey the works to which this paper can be associated. We postpone the analysis of the main differences with these works in section \ref{compare}.

{\bf Wasserstein autoencoders}
A version of autoencoders in conjunction with GANs has been considered in \cite{AAE} and later generalized with Wasserstein distance in \cite{Tolstikhin:aa} and subsequent works. In short, in this approach, one trains an adversarially an autoencoder. Moreover, one adds a regularizer term to get a generative model. More details in section \ref{proof!}.

{\bf Adversarial learned inference}\label{adversarial_related}
A second stream of papers, \cite{ali}, \cite{BiGan}, \cite{hali}, \cite{DistillGAN} and \cite{ARAE} to cite a few, uses another blend of autoencoders with GANs. Their key idea is to learn adversarially an encoder $g_x : \chi \rightarrow M$ together with a decoder $g_m : M \rightarrow \chi$ against a discriminator $c : \chi\times M \rightarrow \mathbb{R}$ in a way that the distributions given by the couples $(g_m(m),m)$ and $(x,g_x(x))$ are indistinguishable from the discriminator point of view. Note that one does not explicitly train $g_x$ and $g_m$ to be inverses to each other. However theorem 2 of ( \cite{BiGan} ) shows that at optimality they are indeed.

 {\bf Adversarial Latent Autoencoders} These autoencoders \cite{ALAE} have the property of learning the latent distribution to match the encoded distribution. This is very different from other traditional approaches that assume an a priori target latent distribution and learn the encoder to match the encoded distribution with this a priori target. From this perspective, this method is very close to the architecture we use in our work, though the objectives are very different : their point is to disentangle representations in order to be able to control the features, whereas our objective is to do transfer. The similarity in architecture probably explains why ALAE autoencoders are very suited to our method.

{\bf Fine-Tuning}
On the side of papers addressing transfer for GANs, we are aware of \cite{Wang:2018aa}, \cite{medical}. Both apply to GANs {\it fine-tuning}, one of the techniques of transfer learning. It consists in initializing the training of a network on a target dataset $\mathcal{D}'$ with weights from another network with the same architecture, but already trained on a similar source dataset $\mathcal{D}$. The two papers seem to have been written independently. While \cite{medical} is mainly targeting a specific application of de-noising in medical imagery, \cite{Wang:2018aa} is rather interested in understanding fine-tuning for GANs per se. Both report a faster convergence and a better quality, though \cite{Wang:2018aa} also observes that fine-tuning enables training with smaller datasets and that the distance between the source and target datasets influences the quality of the training.

\section{Mind to mind algorithm}

We now adapt  to GANs the cut-and-paste procedure described in \ref{Jaimelatrans}. The difference is that in addition to the classifier $c$ (or {\it critic} in the language of WGAN), one also has a generator $g$. Let us consider a factorisation of the form  $g=g_1\circ g_0$ \begin{tikzcd}[column sep=1.5em]
 &M' \arrow{dr}{g_1} &\\
Z\arrow{ur}{g_0} \arrow{rr}{g} && \chi.
\end{tikzcd}

\label{section::algo}
 \begin{algorithm}[H]
    \small
     \caption{Mind2Mind transfer learning.}\label{algo::transfer}
    \begin{algorithmic}
    \Require  $(c_1,g_1)$, an autoencoder trained on a source dataset $\mathcal{D}$, $\alpha$, the learning rate, $b$, the
      batch size, $n$, the number of iterations of the critic
      per generator iteration, $\mathcal{D}'$, a target dataset, $\varphi'$ and $\theta'$ the initial parameters of the critic $c'_0$ and of the generator $g'_0$.
    \State Compute $c_1(\mathcal{D}')$.
    \While{$\theta'$ has not converged}
      \For{$t = 0, ..., n$}
        \State Sample $\{m^{(i)}\}_{i=1}^b \sim {c_1}_\sharp \mathbb{P}_{\mathcal{D}'}$ a batch from $c_1(\mathcal{D}')$.
        \State Sample $\{z^{(i)}\}_{i=1}^b \sim \mathbb{P}_{Z}$ a batch of prior samples.
        \State Update $c'_0$ by descending $L_c$.
      \EndFor
      \State Sample $\{z^{(i)}\}_{i=1}^b \sim \mathbb{P}_{Z}$ a batch of prior samples.
       \State Update $g'_0$ by descending  $-L_g$.
    \EndWhile
    \State \Return $g_1\circ g'_0$.
    \end{algorithmic}
\end{algorithm}

Our algorithm assumes that $g_1$ comes from an autoencoder $(c_1,g_1)$ that has been trained on a source dataset $\mathcal{D}$. The algorithm passes the second data set $\mathcal{D}'$ through the encoder $c_1$ and trains a MindGAN $(g_0',c'_0)$ on the encoded data $c_1(\mathcal{D}')$. One obtains the final generator as the composition $g_1\circ g'_0$ of $g_1$, the decoder of the autoencoder, with $g'_0$, the generator of the MindGAN. We denote by $L_c$ and $L_g$ the losses of the discriminator and the generator.

 In the remainder of the paper, we use for $L_g$ and $L_c$  the losses of a WGAN with gradient penalty \cite{Gulrajani:aa} :       $L_g := \mathbb{E}_{z\sim  \mathbb{P}_Z}  c'_0(g'_0(z)),$  $L_c := - \mathbb{E}_{m\sim {c_1}_\sharp \mathbb{P}_{\mathcal{D}'}} \ c'_0(m) +  L_g +\lambda \mathbb{E}_{m\sim {c_1}_\sharp \mathbb{P}_{\mathcal{D}'}, z\sim  \mathbb{P}_Z, \alpha\sim(0,1)} \lbrace (\Vert \nabla  c'_0(\alpha m +(1-\alpha)g'_0(z))\Vert_2-1)^2\rbrace$.

\begin{rem} This algorithm can be applied (with minor modifications) to conditional GANs. We refer to Appendix \ref{cond} for more details.
\end{rem}

{\bf Motivation for the approach} \label{mind2mind} The architectures of a generator and a critic of a GAN are symmetric to one another. The high-level features appear in $g_0$, the closest to the prior vector (resp. $c_0$, the closest to the prediction), while the low-level features are in $g_1$, the closest to the generated sample (resp $c_1$, the closest to the input image). Therefore, the analogy with cut and paste is to keep $g_1$ (the low level features of $\mathcal{D}$) and only learn the high level features $g_0'$ of the target data set $\mathcal{D}'$. However, the only way a generator can access to information from $\mathcal{D}'$ is through the critic $c$ via the value of $c\circ g=c_0'\circ c_1 \circ g_1 \circ g_0'$  \cite{Gulrajani:aa}. Hence, the information needs to back-propagate through $c_1\circ g_1$ to reach the weights of $g_0'$. Our main idea is to bypass this computational burden and
train directly $g_0'$ and $c_0'$ on $c_1(\mathcal{D}')$. But this requires that the source of $c_0'$ coincides with the target of $g_0'$. Therefore, we assume that $M=M'$, a first hint that autoencoders are relevant for us. 

A second hint comes from an analogy with humans learning a task, like playing tennis, for instance. One can model a player as a function $Z\overset{g}{\rightarrow}\chi$, where $\chi$ is the space of physical actions of the player. His/her coach can be understood as a function $\chi\overset{c}{\rightarrow}\mathbb{R}$, giving $c(g(z))$ as a feedback for an action $g(z)$ of $g$. The objective of the player can be understood as to be able to generate instances of the distribution $\mathcal{D}'$ on $\chi$ corresponding to the ``tennis moves". However, in practice, a coach rarely gives his/her feedback as a score. He instead describes what the player has done and should do instead. We can model this description as a vector $c_1(g(z))$ in $M$, the mind of $c=c_0\circ c_1$. In this analogy, $c_1$ corresponds to the coach analyzing the action, while $c_0$ corresponds to the coach giving a score based on this analysis. One can also decompose the player itself as $g=g_1\circ g_0$. Here $g_0$ corresponds to the player conceiving the set of movements he/she wants to perform and $g_1$ to the execution of these actions. Therefore, two conditions are needed for the coach to help efficiently his/her student : 
\begin{enumerate}
\item they must speak the same language in order to understand one each other,
\item the player must already have a good command of his/her motor system $g_1$.
\end{enumerate}

 In particular, the first constraint implies that they must share the same feature space, i.e., $M=M'$. A way to ensure that both constraints are satisfied is to check whether the player can reproduce a task described by the coach, i.e., that \begin{equation}\label{auto}g_1(c_1(x))=x\end{equation}
 holds. One recognizes in (\ref{auto}) the expression of an autoencoder. It is important to remark that usually, based on previous learning, a player already has a good motor control and he/she and his/her coach know how to communicate together. In other words $g_1$ and $c_1$ satisfy (\ref{auto}) before the training starts. Then the training consists only in learning $g_0'$ and $c_0'$ on the high level feature interpretations of the possible tennis movements, i.e., on $c_1(\mathcal{D}')$.

\section{Theoretical guarantee for convergence}\label{j'te l'garantis}

The following theorem (cf. Appendix \ref{proof!}) enables a very precise control of the convergence of the generated distribution towards the true target distribution. It gives an upper bound on the convergence error $er_{conv}$ in terms of the domain shift $er_{shift}$, the autoencoder quality $er_{AE}$ and the quality of the mindGAN $er_{mind}$.
\begin{theo} \label{controle}
There exist two positive constants $a$ and $b$ such that
\begin{eqnarray}
er_{conv} \leq   a \cdot er_{shift}+ er_{AE} +b \cdot er_{mind}.
\end{eqnarray}
\end{theo}

To be more precise, with the notations $er_{conv}=W(\mathbb{P}_{\mathcal{D}'},\mathbb{P}_\theta')$, $er_{shift}=W(\mathbb{P}_\mathcal{D}, \mathbb{P}_{\mathcal{D}'})$, $er_{AE}=W(AE(\mathbb{P}_\mathcal{D}),\mathbb{P}_\mathcal{D})$ and  $er_{mind}=W({c_1}_{\sharp}\mathbb{P}_{\mathcal{D}'},{g'_0}_\sharp\mathbb{P}_{Z})$. 

Very concretely, theorem \ref{controle} tells us that in order to control the convergence of the transferred GAN towards the distribution of the target dataset $\mathcal{D}'$, we need the exact analogues of steps 1-3 of \ref{steps} : \begin{enumerate}
\item  choose two datasets $\mathcal{D}'$ and $\mathcal{D}$ very similar, i.e., $W(\mathbb{P}_\mathcal{D}, \mathbb{P}_{\mathcal{D}'})$ small,
\item  choose a good autoencoder $(c_1,g_1)$, i.e., $W(\mathbb{P}_\mathcal{D},AE(\mathbb{P}_\mathcal{D}))$ small,
\item  train well the MindGAN $(g_0,c_0)$ on $c_1(\mathcal{D}')$, i.e., $W({c_1}_{\sharp}\mathbb{P}_{\mathcal{D}'},\mathbb{P}'^0_\theta)$ small.
\end{enumerate}

\begin{rem} The main theorem of \cite{sinkhorn}, theorem 3.1, guarantees the convergence of a Wasserstein autoencoder (WAE). We show in Appendix \ref{proof!} that it is a  direct consequence of our theorem.
\end{rem}

\section{Evaluation} \label{results}

{\bf Datasets}. We have tested our algorithm at resolution $28\times 28$ in grey levels (scaled in range $[-1;1]$) on  MNIST \cite{mnist}, KMNIST \cite{kmnist}, FashionMNIST \cite{Fashionmnist}
(60 000 images each) and at resolution $1024\times  1024$ in color on CelabaHQ \cite{growing} (30 000 images) from models trained on the 60 000 first images of FFHQ \cite{StyleGAN} (which consists of 70 000 images), using the {\bf library} Pytorch. The {\bf hardware} for our experiments with $28\times 28$ images consisted of a desktop with 16 Go of RAM and a GPU Nvidia GTX 1080 TI. Most of our experiments with HD color images used a node-gpu \cite{JeanZay} with two CPU Intel Cascade Lake 6248 and a GPU Nvidia V100 SXM2 32 Go. We have also benchmarked the running time on entry level GPU GTX 1060. Despite its limitations \cite{Borji:aa}, we have used {\it FID} (Frechet Inception Distance) \cite{Heusel:a} as {\bf metric}. It is the current standard for evaluations of GANs.   Our {\bf code} is available at \cite{git}.

 {\bf At resolution 28$\times $28}.  The encoder $c_1$ has three convolutional layers with instance normalisation (in) and relu : 32+in+relu, 64+in+relu, 128+in+relu, followed by a single dense layer 256+tanh. The decoder $g_1$ has a single dense layer 4*4*64 + relu and three deconvolutional layers with batch normalisation (bn) : 64+bn+relu, 64+bn+relu, 32+bn+relu, 1+tanh. The MindGAN is a MLP WGAN whose generator $g_0$ consists in three dense layers : 512+bn+relu, 512+bn+relu, 256+tanh and the critic $c_0$ consists also in three dense layers :  256+relu, 256+relu,1. Our {\bf hyper-parameters} :  learning rate of $10^{-3}$,  batch size of 128 for all the networks, 80 epochs for $(g_1,c_1)$ and 100 for the other networks, gradient penalty with $\lambda=10$, beta parameters in Adam optimizer $(.9,.9)$ for $(g_1,c_1)$, $(.1,.5)$ for the other networks.

 We report the results with $\mathcal{D}'=\text{MNIST}$ for $c_1$ trained on each dataset (see  Appendix \ref{Appendice: exp} for other $\mathcal{D}'$). We compare our results to a Vanilla WGAN with architecture ($g_1\circ g_0$, $c_0\circ c_1$), so that the number of parameters agrees, for fair comparison.

{\bf Baseline 1} Vanilla WGAN with gradient penalty trained on MNIST. We have used for the Vanilla WGAN a model similar to the one used in  \cite{Gulrajani:aa}. However, this model would not converge properly, due to a problem of "magnitude race." We have therefore added an $\epsilon$-term \cite{growing}, \cite{Aigner:aa} to ensure its convergence. Our results appear in the first graph on the {\it l.h.s} of figure \ref{figure::graphes MindGAN}, \begin{figure}[t]
    \centering
    \includegraphics[width=.22\textwidth]{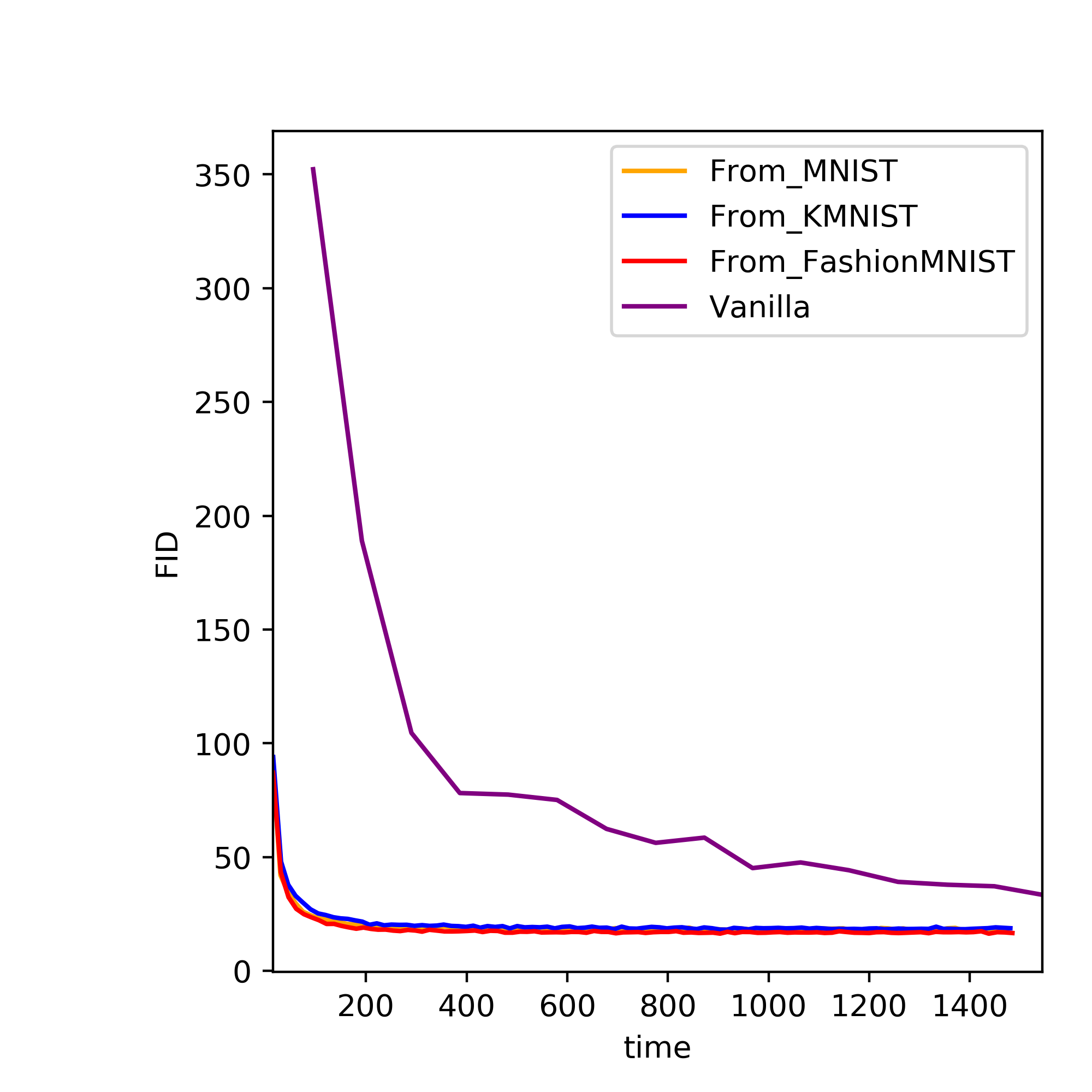}    \includegraphics[width=.22\textwidth]{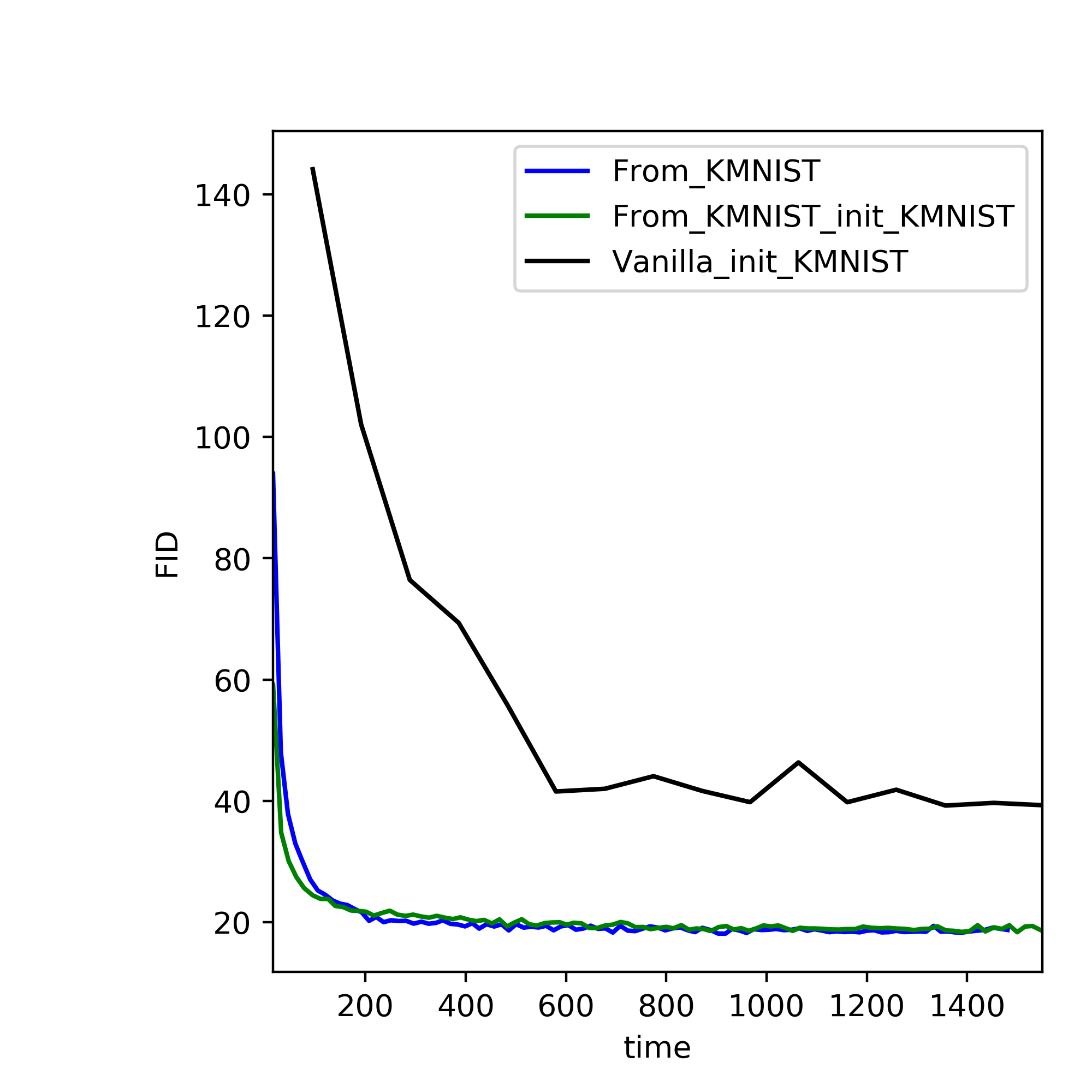} \includegraphics[width=.46\textwidth]{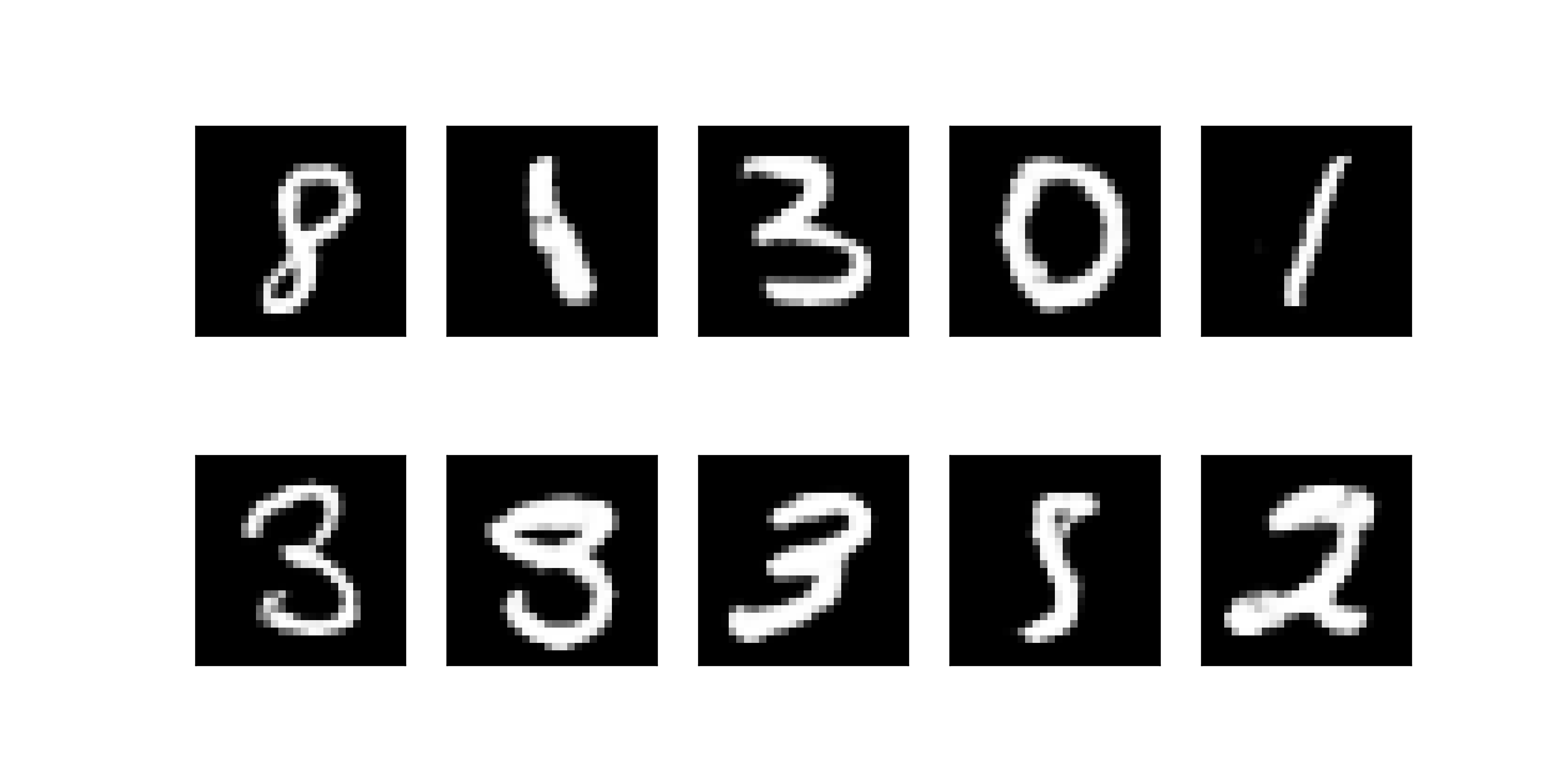}

\medskip
    \caption{
        Mind2Mind training and samples in $28\times 28$.
            }
    \label{figure::graphes MindGAN}
\end{figure} with time in seconds in abscissa. One can observe extremely fast convergence of the mindGAN to good scores, in comparison with the Vanilla WGAN. Note that we have not smoothed any of the curves. This observation suggests that our approach, beyond the gain in training time, provides a regularising mechanism. The stability of the training confirms this hypothesis. Indeed, {\bf statistics over 10 runs} demonstrate a very small standard deviation, as shown in figure \ref{figure::stats} in the supplementary material.  In particular, this regularization enabled us to use a much bigger learning rate ($10^{-3}$ instead of $10^{-4}$), adding to the speed of convergence. In terms of epochs, the MindGAN and the Vanilla WGAN learn similarly (cf Appendix \ref{Appendice: exp}).

{\bf Baseline 2} {\it Fine tuning} studied in \cite{Wang:2018aa}. We have trained a Vanilla WGAN with gradient penalty on $\mathcal{D}=$KMNIST, the dataset the closest to $\mathcal{D}'=$MNIST. We have then {\it fine-tuned} it on $\mathcal{D}'$, i.e., trained a new network on $\mathcal{D}'$, initialized with the weights of this previously trained Vanilla WGAN. 
We display it on the  {\it r.h.s} of figure \ref{figure::graphes MindGAN} 
  under the name {\it Vanilla init Kmnist}, together with our best result, namely a Mind2Mind transfer on $\mathcal{D}'=$MNIST from $\mathcal{D}=$KMNIST. One can observe that the Mind2Mind approach achieves significantly better performances in FID. 
  
   The bottom of figure \ref{figure::graphes MindGAN} displays samples of images produced by a MindGAN trained on MNIST images encoded using a KMNIST autoencoder.
 \begin{figure}[t]
    \centering
    \includegraphics[width=.40\textwidth]{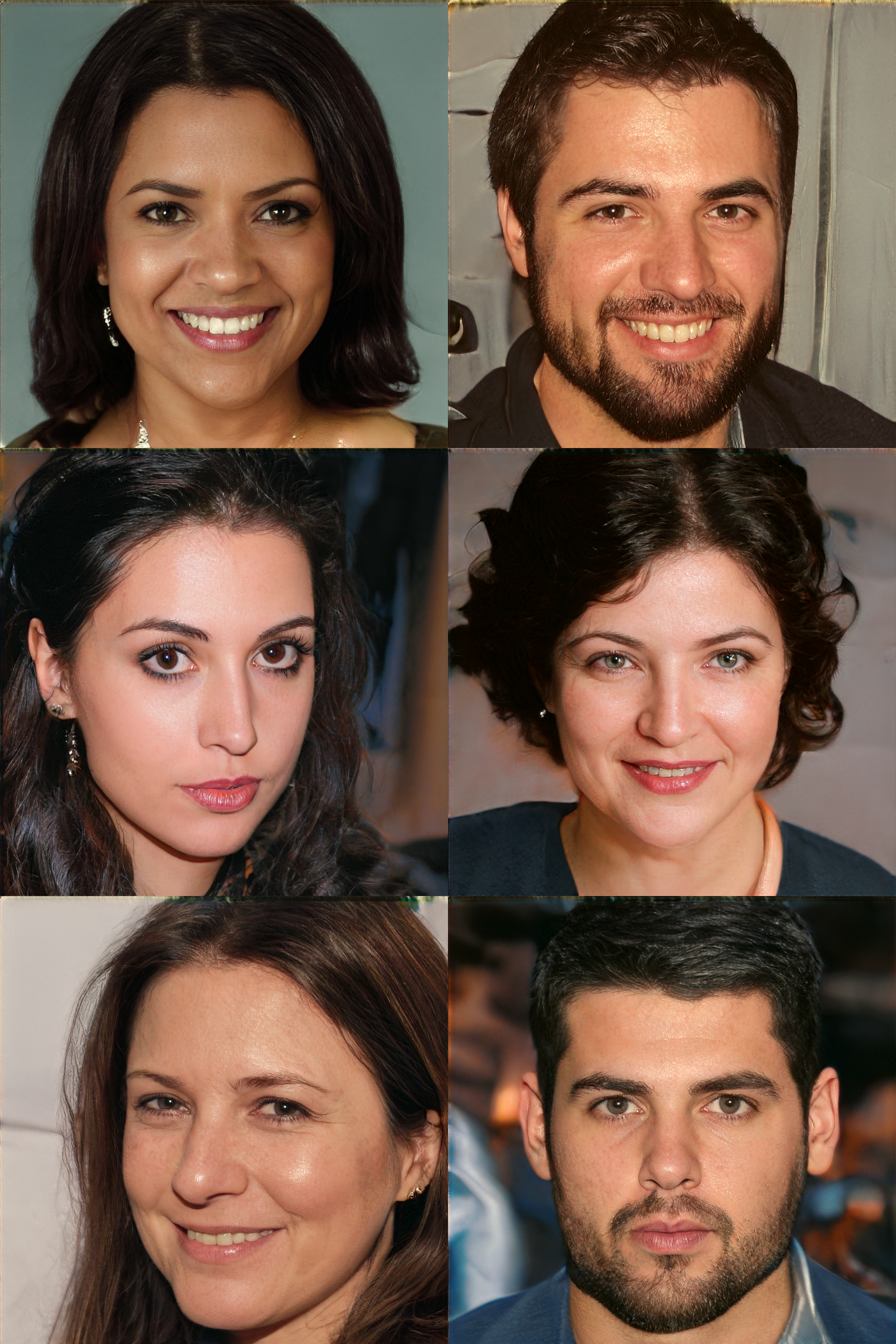} 
\medskip
    \caption{
        Mind2Mind on CelebaHQ transfered from FFHQ.
            }
    \label{figure::mindgan k mnist}
\end{figure}

{\bf At resolution 1024$\times $1024}. We have worked with the encoder and decoder of the ALAE model \cite{ALAE} pre-trained on FFHQ available at \cite{gitALAE}. The generator of our MindGAN has an input dimension of 128,  three hidden dense layers with relu activation (128, 256, 512) followed by a dense output layer with 512 units (no activation). The critic has an input dimension of 512, three hidden dense layers with relu activation (512, 256, 128) followed by a dense output layer with one unit (no activation). Its hyperparameters were lr $= 1e^{-3}$, betas $= [0., 0.5]$, gradient penalty $= 10$, epsilon penalty $= 1e^{-2}$, batch size = 256, critic iteration = 5, epochs = 300.

We have encoded the dataset CelebaHQ \cite{Karras2018ProgressiveGO} and then trained a mindGAN on a V100  for 300 epochs at 16.95 s/epoch during 1h24m.  On a GTX 1060 the same training takes 30.67 s/epoch. We have reached (over 5 runs) an average FID of 15.18 with an average standard deviation of 0.8. This is a better result than the FID score (19.21) of an ALAE directly trained from scratch (see table 5  from \cite{ALAE}).  Samples are displayed on figure \ref{figure::mindgan k mnist}. Compared to the results reported on the ProGAN and StyleGAN official repositories \cite{ProGAN}, \cite{gitStyleGAN}, our training (1.5 hour) on {\bf 1  GPU V100} is roughly {\bf 224 times faster} than the training of a {\bf proGAN} (2 weeks) and {\bf 656 times faster} than the training of a {\bf StyleGAN} (41 days). The training of a MindGAN on a {\bf GTX 1060} is about {\bf  112 times faster} than the training  of a proGAN and {\bf  328 times faster} than the training of a StyleGAN, both on a V100.  Note that a GTX 1060 costs around 200 \$ while a V100 is around 8000 \$. One has to mention however that, on CelebaHQ, the FID of a ProGAN is  8.03 (see table 5  from \cite{ALAE}), while the FID of a StyleGAN is 4.40 (see table 4 of \cite{StyleGAN}), so both are significantly better than ours. We see two factors that can explain the acceleration of the training. The first one is that there are much less parameters to train. Indeed, our mindGAN in HD has around 870K parameters, while the ALAE model (based on a StyleGAN architecture) has 51M parameters. So this already represents a difference of almost two orders of magnitude. One can suspect that the rest of the difference comes from the fact that we bypass 18 layers in the computation of the backpropagation. We believe that the validation of this hypothesis deserves a careful experimental study.

\section{Comparison to other works}\label{compare}

{\bf Wasserstein autoencoders}. WAEs do not provide a solution to the question we address here. Indeed, \cite{Tolstikhin:aa} do not consider at all transfer learning and work only with a single data set $\mathcal{D}$ at a time. Its goal is rather to give a new approach to Variational Auto Encoders based on the use of the Wasserstein distance.

{\bf Adversarial learned inference}. The works on adversarial leaned inference offer an alternative way to train autoencoders. They do not address transfer learning, however, it is possible to choose these types of auto-encoders as building block for our transfer method. We haven't yet conducted experiments with such auto-encoders. In particular, we do not know if theoretical results similar to theorem \ref{controle} can be obtained in this setting.

 {\bf Adversarial Latent Autoencoders} The architecture of this method is very close to the architecture we use in our work, though the objectives are very different : their point is to disentangle representations in order to be able to control the features, whereas our objective is to do transfer. The similarity in architecture probably explains why ALAE are very suited to our method. We did not have enough time to test wether the disentanglement properties of ALAE are preserved via transfer. We plan to investigate further this question.

{\bf Fine-Tuning.} Our approach is different but can be, in theory, combined with fine-tuning. Indeed, one can initialise the training of our MindGAN $(g_0,c_0)$ of algorithm \ref{algo::transfer} on $c_1(\mathcal{D'})$ with another MindGAN trained on $c_1(\mathcal{D})$. We have tried this on the MNISTs datasets; see figure \ref{figure::graphes MindGAN}, but no significant improvement has been observed, We have compared fine-tuning against Mind2Mind transfer for WGANs and report better results in terms of FID. Note that our theorem \ref{controle}  gives a theoretical justification, in our setting, of the observation of \cite{Wang:2018aa} of the influence of the domain shift (distance between $\mathcal{D}$ and $\mathcal{D}'$) on the convergence of the learning.

\section{Disadvantages}

Our first limitation is that we do not learn the transferred layers during the transfer. On the one hand, it is a feature as it enables faster learning. On the other hand, it is possible that at the asymptote, i.e., after the training has converged, the transferred mindGAN offers a worse quality than a Vanilla WGAN. We have not observed this phenomenon with the ALAE architecture. To the contrary, we obtained a better  asymptote since we got a FID of 15.18 for the transferred model  compared to the FID of 19.21 of the ALAE model trained from scratch on CelebaHQ. Maybe this is due to the fact that we have focussed the training on a more significant part of the network (the MindGAN). 

The second limitation of our approach is that it does not provide improvement of the training on limited data. However, such a training was not the objective of our work as we believe that acceleration of the training is a useful goal on itself. 

Our results, compared to state of the art baselines, show a worse performance in terms of quality (FID of 15.18 compared to 8.03 for ProGAN and 4.40 for StyleGAN). However, this should not be an obstacle to the use of this algorithm. Indeed, the targeted users are practitioners without significant computing capacity. Their need is to reach a reasonable quality in a short time. With this regard, in terms of wall clock time, this algorithm learns roughly 224 times faster than a ProGAN and 656 times faster than a StyelGAN on CelebaHQ.  So merits of our algorithm will depend on the tradeoff between the needs  of the users in terms of quality and their constraints in terms of computing capacity.

 The bottleneck in our approach is the lack of a zoo of models of autoencoders trained on diverse datasets in high resolution. We hope that our method, in conjunction with the ALAE architecture, will lead the main players in the industry to train such models and make them accessible to the community.

\section{Conclusion}

We introduced a method that enables transfer learning for GANs. Given an autoencoder trained on a source dataset, one passes the target dataset through the encoder and uses the encoded features to train a GAN, called a MindGAN, in the latent space of the autoencoder. Composing the MindGAN with the decoder provides the transferred GAN, a generator for the target dataset. We provided theoretical results that guaranty that the transferred GAN converges to the target data. We have demonstrated that our method enables  to train GANs much faster (between 6 and 656 times faster, depending on the size of the network) than state of the art methods, in both $28\times 28$ gray scale and $1024\times 1024$ HQ color datasets.

\section*{Broader impact}

The main impact of our work will be a democratization of the use of GANs. Indeed, without the barrier of computing time/costs, GANs for high quality images will become accessible to a pool of practitioners much broader than the researchers/engineers of big industry and academic labs. As a corrollary, one can expect a development of research in this field proportional to the increase of number of researchers who will gain access to these tools.  

Affordable computing time will offer the possibility to customize models for specific needs and datasets. This can lead to new applications in movie industry, virtual assistants or even telecommunications (it may be more efficient to learn a model of a person, send it and use it to reconstruct a video signal than to transmit the signal itself) to name a few. But this can also lead to malicious uses, in particular for the generation fake profiles and deep fakes that can be used for phishing or disinformation.

{\small
\bibliography{biblio}
}

\newpage
\appendix
 \section{Proof of things}\label{proof!}

We first recall notions that will be needed. We then give the proof of our main theoretical result.

\subsection{Wasserstein distance and Lipschitz functions}

In the following, all the metric spaces considered will be subsets of normed vector spaces, with the metric on the subset induced by the norm.

First, one recalls some definitions (more details can be found in \cite{Villani}).
\begin{def*}[transference plan]
    Let $(X, \mathbb{P}_X)$ and $(Y,\mathbb{P}_Y)$ be two probability spaces. A \emph{transference plan $\gamma$} is a measure on $X \times Y $ such that :
    $$
        \int_{A\times Y}d\gamma = \mathbb{P}_X(A),
    $$
    and,
    $$
        \int_{X\times B}d\gamma = \mathbb{P}_Y(B).
    $$
    $\mathbb{P}_X$ and $\mathbb{P}_Y$ are called the \textbf{marginals} of $\gamma.$ The set of transference plans with marginals $\mathbb{P}_X$ and $\mathbb{P}_Y$ is denoted by $\Pi(\mathbb{P}_X,\mathbb{P}_Y)$.
\end{def*}

\begin{def*}[p-Wasserstein distance]
    Let $(X, \Vert . \Vert)$ be a metric space and $p \in [1,+\infty)$. For two probability measure $\mathbb{P}_1$, $\mathbb{P}_2$ on $X$, the \emph{$p$-Wasserstein distance} between $\mathbb{P}_1$ and $\mathbb{P}_2$ is defined by the following
    $$
        W_p(\mathbb{P}_1,\mathbb{P}_2)  = \left( \inf_{\gamma \in \Pi(\mathbb{P}_1,\mathbb{P}_2)} \mathbb{E}_{(x,y)\sim  \gamma}\Vert x -  y\Vert^p \right  )^{\frac{1}{p}}.
    $$

\end{def*}
    
In this paper, we used the notation $W(\mathbb{P}_1,\mathbb{P}_2)$ instead of $W_1(\mathbb{P}_1,\mathbb{P}_2)$.

\begin{def*}[Lipschitz function]
    Let $\phi : X \rightarrow Y$ be a map between metric spaces $X$ and $Y$. It is called a \emph{$C$-Lipschitz} function if there exists a constant $C$ such that :
    $$
        \forall  \text{x and y}  \in X, \Vert \phi(x) - \phi(y) \Vert_Y \leq C \Vert x - y \Vert_X.
    $$
\end{def*}

\begin{lem}\label{lemme_diff}
Let $\phi : X \rightarrow Y$ be a locally Lipschitz map, with $X$ compact, then there exists a constant $C$ such that
$$W_Y({\phi}_\sharp\mu,{\phi}_\sharp\nu)\leq  C W_X(\mu,\nu).$$
\end{lem}

\begin{proof}
Let $\gamma$ be a transference plan realising $W_X(\mu,\nu)$. Define $\gamma ':=(\phi\times \phi)_\sharp \gamma$. One can check that $\gamma '$ defines a transference plan between $\phi_\sharp\mu$ and $\phi_\sharp\nu$.  Therefore, one has the following relation
\begin{eqnarray*}
W_Y({\phi}_\sharp\mu,{\phi}_\sharp\nu) & \leq & \int \Vert  x-y \Vert d\gamma'(x,y)\\
& = &  \int \Vert \phi(x)-\phi(y) \Vert d\gamma(x,y)\\
& \leq & \int C \Vert x-y \Vert d\gamma(x,y)\\
& = & C W_X(\mu,\nu), 
\end{eqnarray*}
where the first inequality comes from the fact that  $\gamma'$ is a transference plan, the first equality from the definition of the push forward of a measure by a map (recalled in section \ref{WGANS}), the last inequality from lemma   \ref{lemme_compact}, and the last equality from the choice of $\gamma$.
\end{proof}

\begin{lem}\label{lemme_compact}
    Let $\phi : X \rightarrow Y$ be a locally Lipschitz map, and $X$ a compact metric space. Then there exists $C$ such that $\phi$ is a $C$-Lipschitz function.
\end{lem}
\begin{proof}
    By definition of a locally Lipschitz map, for all $x$ in $X$, there exists $U_x$ a neighbourhood of $x$ and a constant $C_x$ such that $\phi$ is $C_x$-Lipschitz on $U_x$. \\
    So $\bigcup_{x \in X} U_x$ is a cover of $X$. Since $X$ is compact, there exists a finite set $I$ such that $\bigcup_{i \in I} U_i$ is a cover of $X$. \\
    One can check that $\phi$ is $C$-Lipschitz on $X$, with $C := \max_{i \in I}(C_{x_i})$.
\end{proof}

\subsection{Proof of theorem \ref{controle}}

We can finally turn to the proof of theorem \ref{controle} :

\begin{theo*} 

There exist two positive constants $a$ and $b$ such that
\begin{eqnarray*}
 W(\mathbb{P}_{\mathcal{D}'},\mathbb{P}_\theta')  \leq & a & W(\mathbb{P}_\mathcal{D}, \mathbb{P}_{\mathcal{D}'})+ W(\mathbb{P}_\mathcal{D},AE(\mathbb{P}_\mathcal{D})) \\
                                                    + & b & W({c_1}_{\sharp}\mathbb{P}_{\mathcal{D}'},\mathbb{P}'^0_\theta).
\end{eqnarray*}

\end{theo*}

\begin{proof}

From the triangle inequality property of the Wasserstein metric and the definition of $\mathbb{P}_\theta'$, one has :
\begin{equation*}
W(\mathbb{P}_{\mathcal{D}'},\mathbb{P}_\theta') \leq  W(\mathbb{P}_{\mathcal{D}'},AE(\mathbb{P}_{\mathcal{D}'}))+W(AE(\mathbb{P}_{\mathcal{D}'}),{g_1}_\sharp\mathbb{P}'^0_\theta).
\end{equation*}
 One concludes with lemma \ref{lemma::AE}  and lemma \ref{lemme_diff} with $\phi=g_1$.
\end{proof}

\begin{lem}\label{lemma::AE}
There exist a positive constant $a$ such that
$$W(\mathbb{P}_{\mathcal{D}'},AE(\mathbb{P}_{\mathcal{D}'}))  \leq a W(\mathbb{P}_\mathcal{D}, \mathbb{P}_{\mathcal{D}'})+ W(\mathbb{P}_\mathcal{D},AE(\mathbb{P}_\mathcal{D})).$$ 
\end{lem}

\begin{proof}
Applying twice the triangle inequality, one has :
\begin{eqnarray*}
W(\mathbb{P}_{\mathcal{D}'},AE(\mathbb{P}_{\mathcal{D}'})) & \leq &W(\mathbb{P}_{\mathcal{D}'},\mathbb{P}_{\mathcal{D}}) +W(\mathbb{P}_{\mathcal{D}},AE(\mathbb{P}_{\mathcal{D}}))\\
                                                                                                   & + & W(AE(\mathbb{P}_{\mathcal{D}}),AE(\mathbb{P}_{\mathcal{D}'})).
\end{eqnarray*}
One concludes with lemma \ref{lemme_diff} with $\phi=g_1\circ c_1$.
\end{proof}

\begin{rem}
It is important to remark that in order to be able to apply lemma \ref{lemme_diff} in the proof of theorem \ref{controle},  one needs the assumption that $\mathbb{P}'^0_\theta$ and ${c_1}_\sharp\mathbb{P}_{\mathcal{D}'}$ have compact support. But as $\chi$ is itself compact, this is not a problem for ${c_1}_\sharp\mathbb{P}_{\mathcal{D}'}$ since the image of a compact $\chi$ by a continuous function $c_1$ is compact. However,  the compacity of the support of $\mathbb{P}'^0_\theta$ is not a priori granted. An easy fix is to choose a prior $\mathbb{P}_Z$ with compact support. Therefore, we choose this setting in our applications. 
\end{rem}

\begin{rem}
Our proof of theorem \ref{controle} implicitly assumed that neural networks are locally Lipschitz maps (see lemmata  \ref{lemme_diff} and \ref{lemma::AE}). This assumption is justified by the following lemma.
\end{rem}

\begin{lem*}
    Let $g : Z \rightarrow X$ be a neural network and $\mathbb{P}_{Z}$ a prior over $Z$ such that $\mathbb{E}_{z \sim \mathbb{P}_{Z}}(\Vert z \Vert )< \infty $ (such as Gaussian) then $g$ is locally Lipschitz and $\mathbb{E}_{z \sim \mathbb{P}_{Z}}(L_z)< \infty $, where $L_z$ are the local Lipschitz constants. 
\end{lem*}

\begin{proof}
 See Corollary 1. of \cite{Arjovsky:2017aa}
\end{proof}

\subsection{Application to Wasserstein autoencoders.} The main theorem of \cite{sinkhorn}, theorem 3.1, guarantees the convergence of a Wasserstein autoencoder (WAE). We recall this theorem and show that it is a direct consequence of our theorem \ref{controle}.

\begin{theo} \label{theo_sink}
\begin{eqnarray}
W(\mathbb{P}_{\mathcal{D}},{g_1}_{\sharp} \mathbb{P}_\mathcal{Z}) \leq & W(\mathbb{P}_\mathcal{D}, AE(\mathbb{P}_\mathcal{D})) \\
                                                            + & bW({c_1}_{\sharp}\mathbb{P}_{\mathcal{D}},\mathbb{P}_\mathcal{Z}).
\end{eqnarray}
\end{theo}

\begin{proof}
Since WAEs do not involve transfer, one has $\mathcal{D}=\mathcal{D}'$, i.e. $\mathbb{P}_\mathcal{D}=\mathbb{P}_{\mathcal{D}'}$.
Then it suffices to replace $\mathbb{P}'^0_\theta$ by $\mathbb{P}_\mathcal{Z}$ and $\mathbb{P}_\theta'$ becomes ${g_1}_{\sharp} \mathbb{P}_\mathcal{Z}$.
\end{proof}

\begin{rem}
Our proof of theorem \ref{controle} is very similar to the proof of theorem \ref{theo_sink} given in \cite{sinkhorn}. Therefore, our contribution here consists rather in finding a versatile statement that applies to both problems (transfer and WAE) than in the originality of the tools used in the proofs.
\end{rem}

\begin{rem}
 When one restricts our approach to the case when $\mathcal{D}=\mathcal{D}'$, it does not coincide with WAE. Indeed, with the notations of our paper, WAE work with a fixed prior $\mathbb{P}_M$ on $M$ that one tries to approximate by ${c_1}_\sharp \mathbb{P}_\mathcal{D}$, while constraining $c_1$ to be a right inverse (in measure) of $g_1$, and ${g_1}_\sharp \mathbb{P}_M$ to approximate (in measure) $\mathbb{P}_\mathcal{D}$. On the other hand, our approach involves an extra auxiliary latent space $Z$. Therefore we can consider ${g_0}_\sharp \mathbb{P}_Z$ as a replacement of  $\mathbb{P}_M$. Via the flexibility of the learnable weights of $g_0$, we use ${g_0}_\sharp\mathbb{P}_Z$ to approximate ${c_1}_\sharp \mathbb{P}_\mathcal{D}$, instead of using ${c_1}_\sharp \mathbb{P}_\mathcal{D}$ to approximate $\mathbb{P}_M$ as in \cite{AAE}. This is fundamental, because in a setting where $\mathcal{D}\neq \mathcal{D}'$, this decoupling permits to train $c_1$ and $g_1$ on $\mathcal{D}$ and $c_0$ and $g_0$ on $c_1(\mathcal{D}')$, enabling us to do transfer.
\end{rem}

\section{Mind2Mind conditional GANs} \label{cond}

As suggested to us by L. Cetinsoy, the Mind2Mind approach also applies to conditional GANs. However, one needs to implement the following modifications : replace $M$ by $M\times L$ and $Z$ by $Z\times L$ in the diagram

 \begin{equation}
\begin{tikzcd}[baseline=(current  bounding  box.center)]
 &M \arrow{dr}{g_1} && M \arrow{dr}{c_0}\\
Z\arrow{ur}{g_0} \arrow{rr}{g} && \chi \arrow{ur}{c_1} \arrow{rr}{c} && \mathbb{R},
\end{tikzcd}
\end{equation}
where L stands for the space of conditions, in order to get

\begin{minipage}{.45\textwidth}

 \begin{equation}
\begin{tikzcd}[column sep=1.5em]
 &M\times L \arrow{dr}{g_1\times \mathbb{I}_L} && M\times L \arrow{dr}{c^c_0}\\
Z\times L\arrow{ur}{g^c_0} \arrow{rr}{g^c} && \chi \times L\arrow{ur}{c_1\times \mathbb{I}_L} \arrow{rr}{c^c} && \mathbb{R}.
\end{tikzcd}
\end{equation}
\end{minipage}

Here, $(g^c_0,c^c_0)$ and $(g^c,c^c)$ are conditional GANs, with the generators of the form $g^c_0(z,l)=(m(z,l),l)$ and $g^c(z,l)=(x(z,l),l)$. The autoencoder $(c_1\times \mathbb{I}_L,g_1\times \mathbb{I}_L)$ can be trivially deduced from an autoencoder $(c_1,g_1)$ via the formulas $c_1\times \mathbb{I}_L(x,l):=(c_1(x),l)$ and  $g_1\times \mathbb{I}_L(m,l):=(c_1(m),l)$.\\
In practice, the algorithm \ref{algo::transfer} becomes a classical conditional GAN algorithm :\\

 \begin{algorithm}[H]
	\small
	 \caption{Conditional-MindGAN transfer learning.}\label{algo::transfer condition}
              \begin{algorithmic}
    \Require $(c_1,g_1)$, an autoencoder trained on a source dataset $\mathcal{D}$, $\alpha$, the learning rate, $b$, the
      batch size, $n$, the number of iterations of the critic
      per generator iteration, $\mathcal{D}'\subset \chi\times L$, a dataset with conditions, $\varphi'$ and $\theta'$ the initial parameters of the critic $c^c_0$ and of the generator $g^c_0$.
    \State Compute $(c_1\times \mathbb{I}_L)(\mathcal{D}')$.
    \While{$\theta'$ has not converged}
      \For{$t = 0, ..., n_{\text{critic}}$}
        \State Sample $\{(m^{(i)},l^{(i)})\}_{i=1}^b \sim {(c_1 \times \mathbb{I}_L)}_\sharp \mathbb{P}_{\mathcal{D}'}$ a batch from $(c_1 \times \mathbb{I}_L)(\mathcal{D}')$.
        \State Sample $\{(z^{(i)},l^{(i)})\}_{i=1}^b \sim \mathbb{P}_{Z \times L}$ a batch of prior samples with conditions.
        \State Update $c^c_0$ by descending $L_c$.
      \EndFor
      \State Sample$\{(z^{(i)},l^{(i)})\}_{i=1}^b \sim \mathbb{P}_{Z \times L} $ a batch of prior samples with conditions.
       \State Update $g^c_0$ by descending  $-L_g$.
    \EndWhile
    \State \Return $g_1\circ g^c_0$.
\end{algorithmic}
        \end{algorithm}

\section{Supplementary experiments}\label{Appendice: exp}

 In figure \ref{figure::more_tuches} we display additional samples from a MindGAN on CelebaHQ transferred from FFHQ.  We also display in figure  \ref{figure::stats} the mean and standard deviation over 10 runs of the training of a MindGAN in $28\times 28$.   The source dataset in figure \ref{figure::mindgan f mnist} is $\mathcal{D}'=$ FashionMNIST, while in figure \ref{figure::mindgan mnist}, $\mathcal{D}'=$ MNIST. For comparison, we display in figure \ref{figure::vanilla mnist} samples from a vanilla WGAN.

\begin{figure}[t]
    \centering
    \includegraphics[width=.40\textwidth]{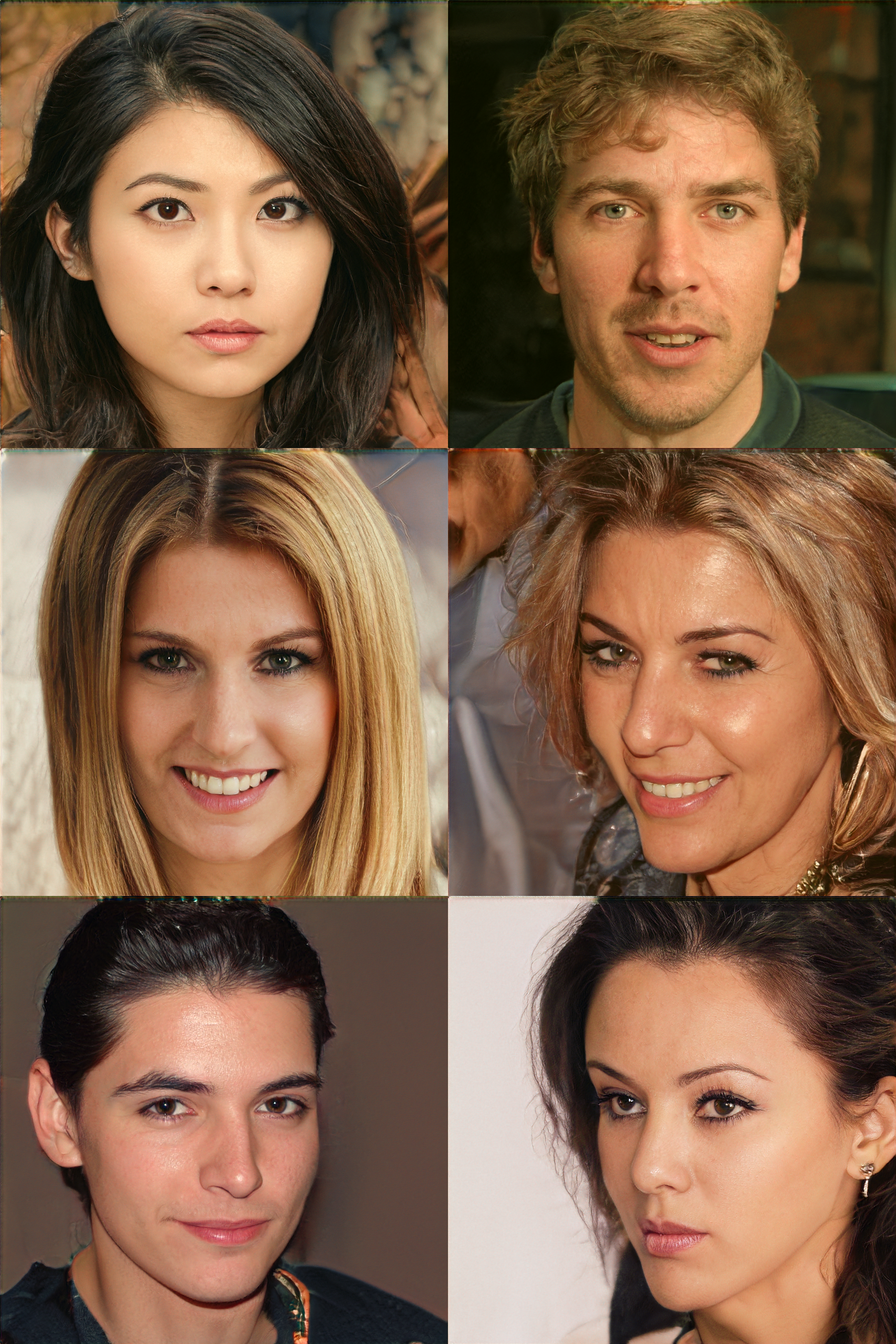} 
\medskip
    \caption{
        Mind2Mind on CelebaHQ transfered from FFHQ .
            }
    \label{figure::more_tuches}
\end{figure}

\begin{figure}[t]
    \centering
    \includegraphics[width=.45\textwidth]{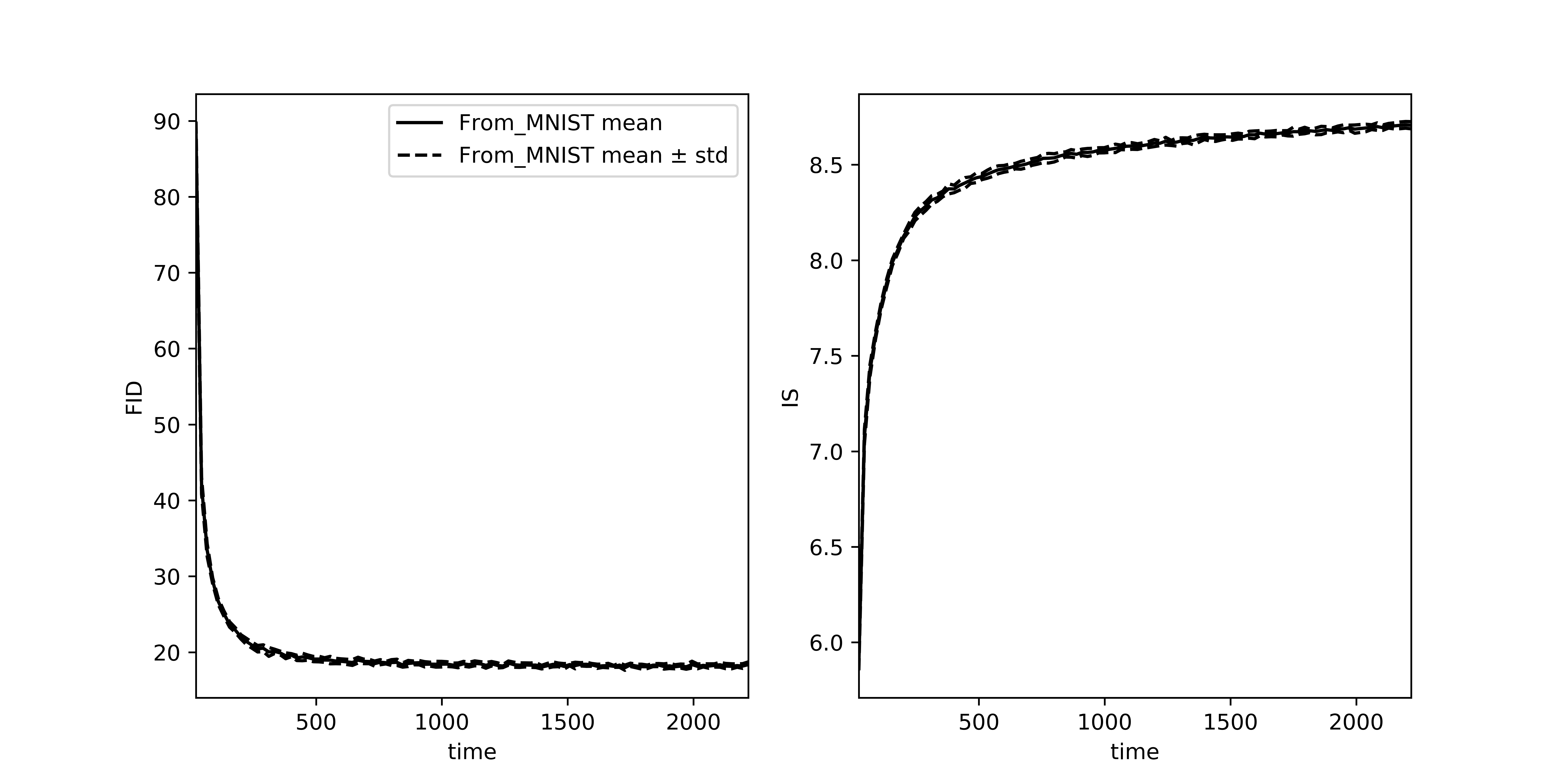} 
\medskip
    \caption{
        Mean and standard deviation of the training of a MindGAN.
            }
    \label{figure::stats}
\end{figure}

 \begin{figure}[t]
    \centering
    \includegraphics[width=.45\textwidth]{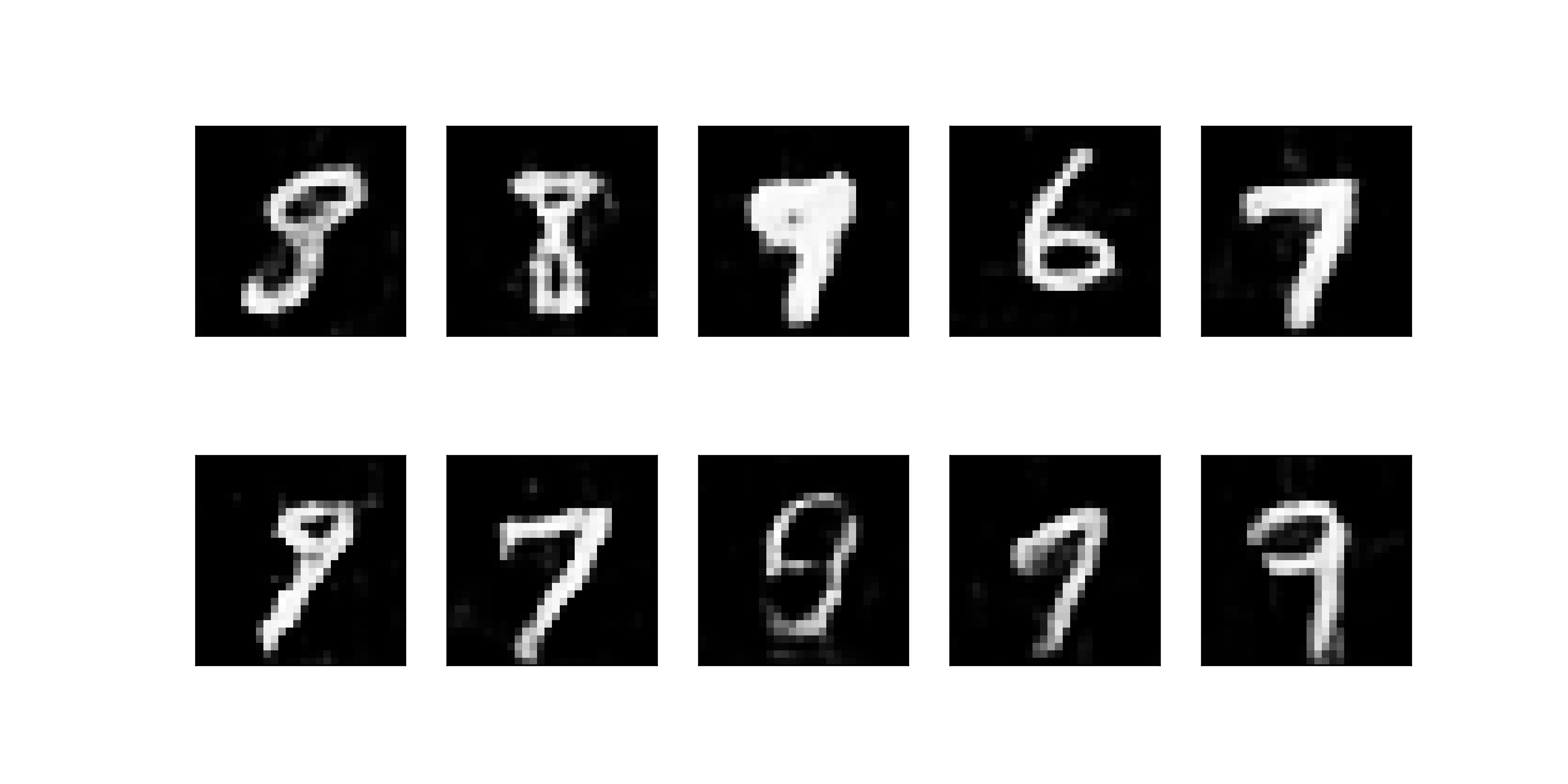}
\medskip
    \caption{
        MindGAN from FashionMNIST.
            }
    \label{figure::mindgan f mnist}
\end{figure}

 \begin{figure}[t]
    \centering
    \includegraphics[width=.45\textwidth]{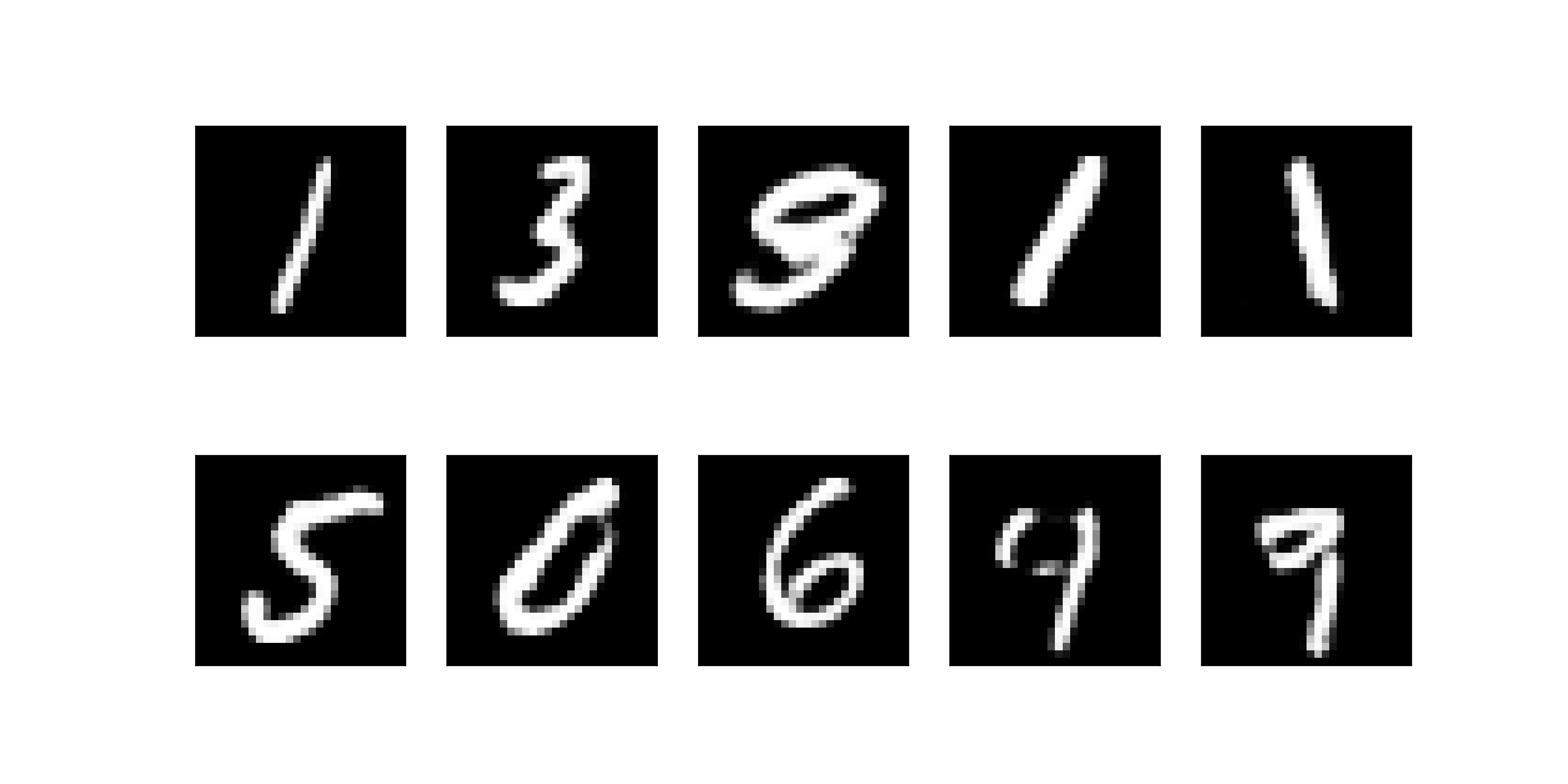}
\medskip
    \caption{
        MindGAN from MNIST.
            }
    \label{figure::mindgan mnist}
\end{figure}

 \begin{figure}[t]
    \centering
    \includegraphics[width=.45\textwidth]{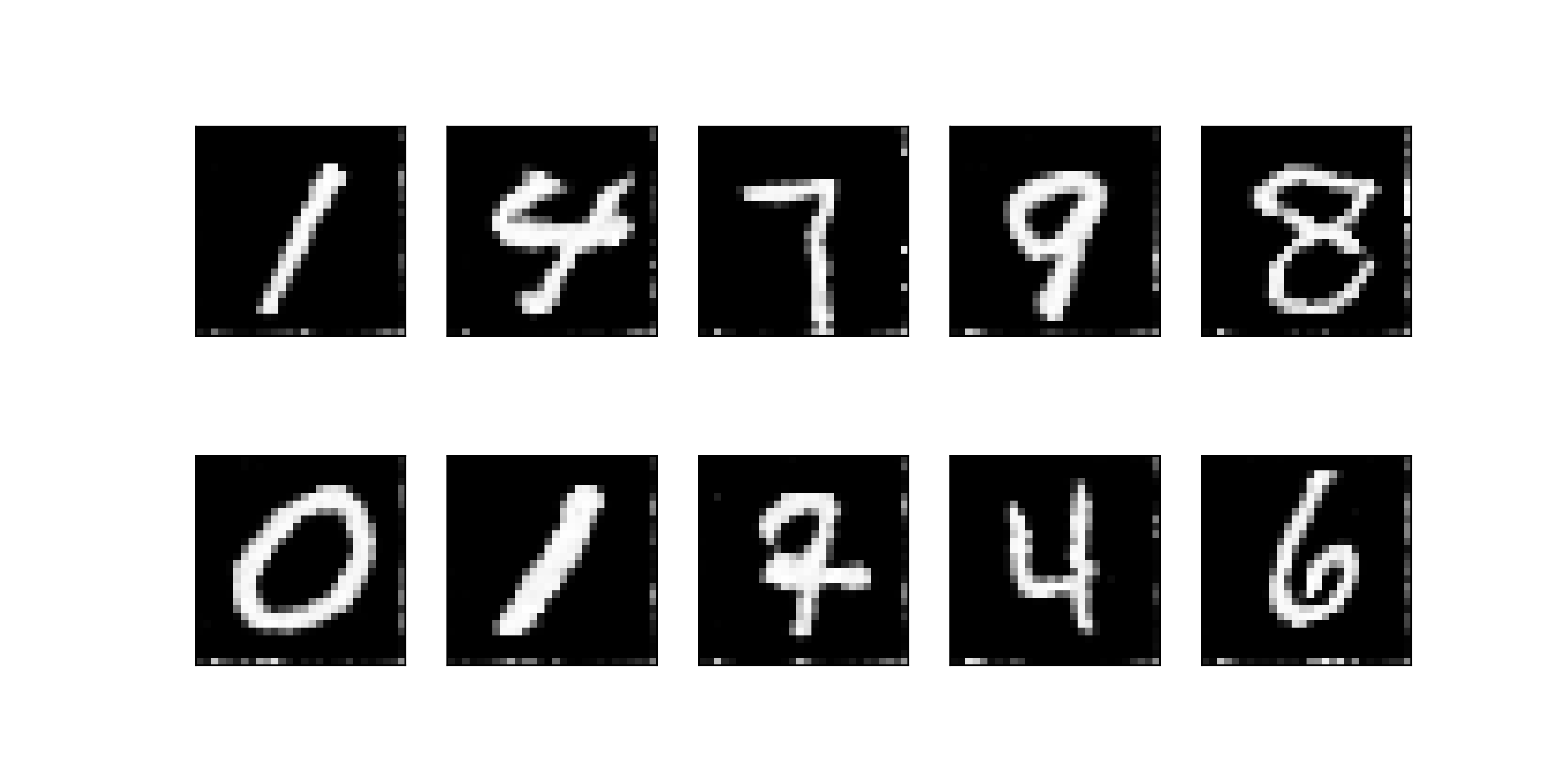}
\medskip
    \caption{
        Vanilla WGAN.
            }
    \label{figure::vanilla mnist}
\end{figure}
\newpage

\end{document}